\documentclass{article}

\usepackage{graphicx} 
\usepackage{subfigure} 

\usepackage{natbib}

\usepackage{algorithm}
\usepackage{algorithmic}

\usepackage{hyperref}

 \usepackage[accepted]{icml2010}

\usepackage{tabularx}
\usepackage{amsmath, amssymb}
\usepackage{dsfont}
\usepackage{theorem}
\newtheorem{proposition}{Proposition}
\newtheorem{lemma}{Lemma}

\newenvironment{proof}[1][Proof]{\begin{trivlist}
  \item[\hskip \labelsep {\bfseries #1}]}{\end{trivlist}}
\newcommand{\qed}{~~~\ensuremath{\blacksquare}}

\def\T{{\cal T}}
\def\R{\mathbb{R}}
\def\E{\mathbb{E}}
\newcommand{\spn}[1]{\mbox{\bf span}\left(#1\right)}

\def\vopt{{\hat v}_{best}}

\icmltitlerunning{TD or BR? The unified oblique projection  view }

\begin{document} 

\twocolumn[
\icmltitle{Should one compute the Temporal Difference fix point or minimize the Bellman Residual~? The unified oblique projection view}

\icmlauthor{Bruno Scherrer}{scherrer@loria.fr}
\icmladdress{LORIA - INRIA Lorraine - Campus Scientifique - BP 239 \\
54506 Vand\oe{}uvre-l{\`e}s-Nancy CEDEX \\
FRANCE}

\icmlkeywords{Stochastic optimal control, Reinforcement Learning,
	Dynamic Programming, Markov Decision Processes}

\vskip 0.3in
]

\begin{abstract} 
  We investigate projection methods, for evaluating a linear
  approximation of the value function of a policy in a Markov Decision
  Process context. We consider two popular approaches, the one-step
  Temporal Difference fix-point computation (TD(0)) and the Bellman
  Residual (BR) minimization.  We describe examples, where each
  method outperforms the other. We highlight a simple relation
  between the objective function they minimize, and show that while BR
  enjoys a performance guarantee, TD(0) does not in general. We then
  propose a unified view in terms of oblique projections of the
  Bellman equation, which substantially simplifies and extends the
  characterization of \citet{schoknecht2002} and the recent analysis
  of \citet{yu}.  Eventually, we describe some simulations that
  suggest that if the TD(0) solution is usually slightly better than
  the BR solution, its inherent numerical instability makes it very
  bad in some cases, and thus worse on average.
\end{abstract}

\section*{Introduction}

We consider linear approximations of the value function
of the policy in the framework of Markov Decision Processes (MDP). We focus on two popular
methods: the {\bf computation of the projected Temporal Difference fixed
point} (TD(0), TD for short), which \citet{antos,fara,sutton09} have recently presented as the
minimization of the mean-square projected Bellman Equation, and the
{\bf minimization of the mean-square Bellman Residual} (BR).
%
In this article, we present some new analytical and empirical data,
that shed some light on both approaches. The paper is organized as
follows. Section \ref{framework} describes the MDP linear
approximation framework and the two projection methods. Section
\ref{examples} presents small MDP examples, where each method
outperforms the other.  Section \ref{stability} highlights a simple
relation between the quantities TD and BR optimize, and show that
while BR enjoys a performance guarantee, TD does not in general.
Section \ref{unified} contains the main contribution of this paper: we
describe a unified view in terms of oblique projections of the Bellman
equation, which simplifies and extends the characterization of
\citet{schoknecht2002} and the recent analysis of \citet{yu}.
Eventually, Section \ref{experiments} presents some simulations, that
address the following practical questions: which of the method gives
the best approximation? and how useful is our analysis
for selecting it a priori?

\section{Framework and Notations}

\label{framework}

\paragraph{The model}
We consider an MDP with a fixed policy, that is an uncontrolled
discrete-time dynamic system with instantaneous rewards. We assume
that there is a {\bf state space} $X$ of finite size $N$. When at
state $i \in \{1,..,N\}$, there is a {\bf transition probability}
$p_{ij}$ of getting to the next state $j$.  Let $i_k$ the state of the
system at time $k$.  At each time step, the system is given a
reward $\gamma^k r(i_k)$ where $r$ is the instantaneous {\bf reward
  function}, and $0<\gamma<1$ is a {\bf discount factor}.  The {\bf
  value} at state $i$ is defined as the total expected return:
$
v(i):=\lim_{N\rightarrow \infty} \E \left[ \left. \sum_{k=0}^{N-1}\gamma^k r(i_k) \right| i_0=i\right].
$
We write $P$ the $N \times N$ stochastic matrix whose elements are
$p_{ij}$.  $v$ can be seen as a vector of $\R^N$.
$v$ is known to be the unique fixed point of the Bellman operator: $\T v
:= r + \gamma P v$, that is $v$ solves the Bellman Equation $v=\T v$ and is equal to $L^{-1}r$ where $L=I-\gamma P$.

\paragraph{Approximation Scheme}
When the size $N$ of the state space is large, one usually comes down
to solving the Bellman Equation approximately. One possibility is to
look for an approximate solution $\hat{v}$ in some specific small
space. The simplest and best understood choice is a linear
parameterization: $\forall $i$,~\hat v(i)=\sum_{j=1}^{m} w_j \phi_j(i)$ where $m \ll N$, the $\phi_j$ are some feature functions that should capture the general shape of $v$, and $w_j$ are the weights that characterize the approximate value $\hat v$. 
For all $i$ and $j$, write $\phi_j$ the $N$-dimensional vector corresponding to the $j^{th}$ feature function and $\phi(i)$ the $m$-dimensional vector giving the features of state $i$. For any vector of matrix $X$, denote $X'$ its transpose. The following $N \times m$ {\bf feature} matrix
$
\Phi = \left( \phi_1 \dots \phi_m \right) = \left( \phi(i_1) \dots \phi(i_N) \right)'
$
leads to write the parameterization of $v$ in a condensed matrix form:
$\hat v=\Phi w$, where $w=(w_1,...,w_m)$ is the $m$-dimensional
{\bf weight} vector. We will now on denote $\spn{\Phi}$ this subspace of $\R^N$ and assume that the vectors $\phi_1,...,\phi_m$ form a linearly independent set.

Some approximation $\hat v$ of $v$ can be obtained by
minimizing $\hat v \mapsto \|\hat v - v\|$ for some norm $\|\cdot\|$, that is
equivalently by projecting $v$ onto $\spn{\Phi}$ orthogonally with respect to $\|\cdot\|$.  In a very general way, any symmetric
positive definite matrix $Q$ of $\R^N$ induces a quadratic
norm $\|\cdot\|_Q$ on $\R^N$ as follows: $\|v\|_Q=\sqrt{v'Qv}$.
It is well known that the orthogonal projection with respect to such a
norm, which we will denote $\Pi_{\|\cdot\|_Q}$, has the following closed form:
$\Pi_{\|\cdot\|_Q}=\Phi \pi_{\|\cdot\|_Q}$ where $\pi_{\|\cdot\|_Q}=(\Phi' Q
\Phi)^{-1} \Phi' Q$ is the linear application from $\R^N$ to $\R^m$
that returns the coordinates of the projection of a point in the basis
$(\phi_1,\dots,\phi_m)$. With these notations, the following relations $\pi_{\|\cdot\|_Q}\Phi=I$ and $\pi_{\|\cdot\|_Q}\Pi_{\|\cdot\|_Q}=\pi_{\|\cdot\|_Q}$ hold.

In an MDP approximation context, where one is modeling a stochastic
system, one usually considers a specific kind of norm/projection.  Let
$\xi=(\xi_i)$ be some distribution on $X$ such that $\xi>0$ (it
assigns a positive probability to all states). Let $\Xi$ be the
diagonal matrix with the elements of $\xi$ on the diagonal. Consider
the orthogonal projection of $\R^N$ onto the feature space
$\spn{\Phi}$ with respect to the $\xi$-weighted quadratic norm
$\|v\|_\xi=\sqrt{\sum_{j=1}^{N}\xi_i {v_i}^2}=\sqrt{v' \Xi v}$.  For
clarity of exposition, we will denote this specific projection $\Pi:=\Pi_{\|\cdot\|_\Xi}=\Phi \pi$ where
$\pi:=\pi_{\|\cdot\|_\Xi}=(\Phi' \Xi \Phi)^{-1} \Phi' \Xi$.

Ideally, one would like to compute the ``best'' approximation
$$
\vopt=\Phi w_{best}\mbox{ with }w_{best} = \pi v = \pi L^{-1}r.
$$
This can be done with algorithms like
TD($1$) / LSTD($1$)\cite{bertsekas,boyan02}, but they require simulating infinitely long
trajectories and usually suffer from a high variance. The
projections methods, which we focus on in this paper, are alternatives
that only consider \emph{one-step} samples.


\paragraph{TD(0) fix point method}
The principle of the TD(0) method (TD for short) is to look for a fixed point of $\Pi \T$, that is, one looks for $\hat{v}_{TD}$ in the space $\spn{\Phi}$ satisfying $\hat{v}_{TD}=\Pi \T \hat{v}_{TD}$. Assuming that the matrix inverse below exists\footnote{This is not necessary the case, as the forthcoming Example 1 (Section \ref{examples}) shows.}, it can be proved\footnote{Section \ref{unified} will generalize this derivation.} that $\hat{v}_{TD}=\Phi w_{TD}$ with
\begin{equation}
\label{td}
w_{TD}=(\Phi'\Xi L \Phi)^{-1}\Phi' \Xi r
\end{equation}
As pointed out by \citet{antos,fara,sutton09}, when the inverse exists, the above computation is equivalent to minimizing for $\hat v \in \spn{\Phi}$ the TD error $E_{TD}(\hat v):=\|\hat v-\Pi \T \hat v\|_{\xi}$ down to 0\footnote{This remark is also true if we replace $\|\cdot\|_\xi$ by any equivalent norm $\|\cdot\|$. This observation lead \citet{sutton09} to propose original off-policy gradient algorithms for computing the TD solution.}.

%

\paragraph{BR minimization method}
The principle of the Bellman Residual (BR) method is to look for $\hat{v} \in \spn{\Phi}$ so that it minimizes the norm of the Bellman Residual, that is the quantity $E_{BR}(\hat v):=\|\hat v-\T \hat v\|_{\xi}$. Since $\hat{v}$ is of the form $\Phi w$, it can be seen that $E_{BR}(\hat v)=\|\Phi w - \gamma P \Phi w - r \|_\xi=\|\Psi w -r\|_\xi$ using the notation $\Psi=L\Phi$. Using standard linear least squares arguments, one can see that the minimum BR is obtained for $\hat{v}_{BR}=\Phi w_{BR}$ with
\begin{equation}
\label{br}
w_{BR}= (\Psi' \Xi \Psi)^{-1}\Psi' \Xi r.
\end{equation}
Note that in this case, the above inverse always exists \citep{schoknecht2002}.

\section{Two simple examples}

\label{examples}

\paragraph{Example 1} Consider the 2 state MDP such that 
$
P=
{\tiny \left(
\begin{array}{cc}
0 &1 \\
0 & 1
\end{array}
\right)}
$.
Denote the rewards $r_1$ and $r_2$. One thus have
$v(1)=r_1+\frac{\gamma r_2}{1-\gamma}$ and $v(2)=\frac{r_2}{1-\gamma}$. Consider the one-feature linear approximation with $\Phi=(1 ~ 2)'$, with uniform distribution $\xi=(.5~ .5)'$.
$\Phi' \Xi \Phi=\frac{5}{2}$, therefore $\pi=\left( \frac{1}{5} \frac{2}{5} \right)$, and the weight of the best approximation is $w_{best}=\pi v=\frac{1}{5}r_1+\frac{2+\gamma}{5(1-\gamma)}r_2$. 
This example has been proposed by \citet{bertsekas} in order to show that fitted Value Iteration can diverge if the samples are not generated by the stationary distribution of the policy. In \cite{bertsekas}, the authors only consider the case $r_1=r_2=0$ so that this diverging result was true even though the exact value function $v(0)=v(1)=0$ did belong to the feature space. In the case $r_1=r_2=0$, the TD and BR methods do calculate the exact solution (we will see later that this is indeed a general fact when the exact value function belongs to the feature space). We thus extend this model by taking $(r_1,r_2)\neq (0,0)$. As a scaling of the reward is translated exactly in the approximation, we consider the general form $(r_1,r_2)=(\cos \theta, \sin \theta)$. 

Consider the TD solution: one has $\Phi' \Xi=\left( \frac{1}{2} ~ 1 \right)$, $(I-\gamma P)\Phi=\left( 1-2\gamma ~ 1-\gamma \right)$, thus $(\Phi'\Xi \Psi)=\frac{5}{2}-3\gamma$ and $\Phi' \Xi r=\frac{r_1}{2}+r_2$. Eventually the weight of the TD approximation is $w_{TD}=\frac{r_1+2r_2}{5-6\gamma}$. One notices here that the value $\gamma=5/6$ is singular.
Now, consider the BR solution. One can see that $(\Psi' \Xi \Psi)^{-1}=\frac{(1-2\gamma)^2+(2-2\gamma)^2}{2}$ and $\Psi' \Xi r=\frac{(1-2\gamma)r_1+(2-2\gamma)r_2}{2}$. Thus, the weight of the BR approximation is $w_{BR}=\frac{(1-2\gamma)r_1+(2-2\gamma)r_2}{(1-2\gamma)^2+(2-2\gamma)^2}$.

For all these approximations, one can compute the squared error $e$ with respect to the optimal solution $v$: For any weight $w \in \{w_{best},w_{TD},w_{BR}\}$, $e(w)=\|v-\Phi w\|_\xi^2=\frac{1}{2}(v(1) - w)^2+\frac{1}{2}(v(2)-2w)^2$. In Figure \ref{ex1}, we plot the squared error ratios $\frac{e(w_{TD})}{e(w_{best})}$ and $\frac{e(w_{BR})}{e(w_{best})}$ on a log scale (they are by definition greater than $1$) with respect to $\theta$ and $\gamma$.
\begin{figure}
\begin{minipage}[l]{.54\linewidth}
\begin{center}
\hspace{-5mm}
\includegraphics[width=4.5cm]{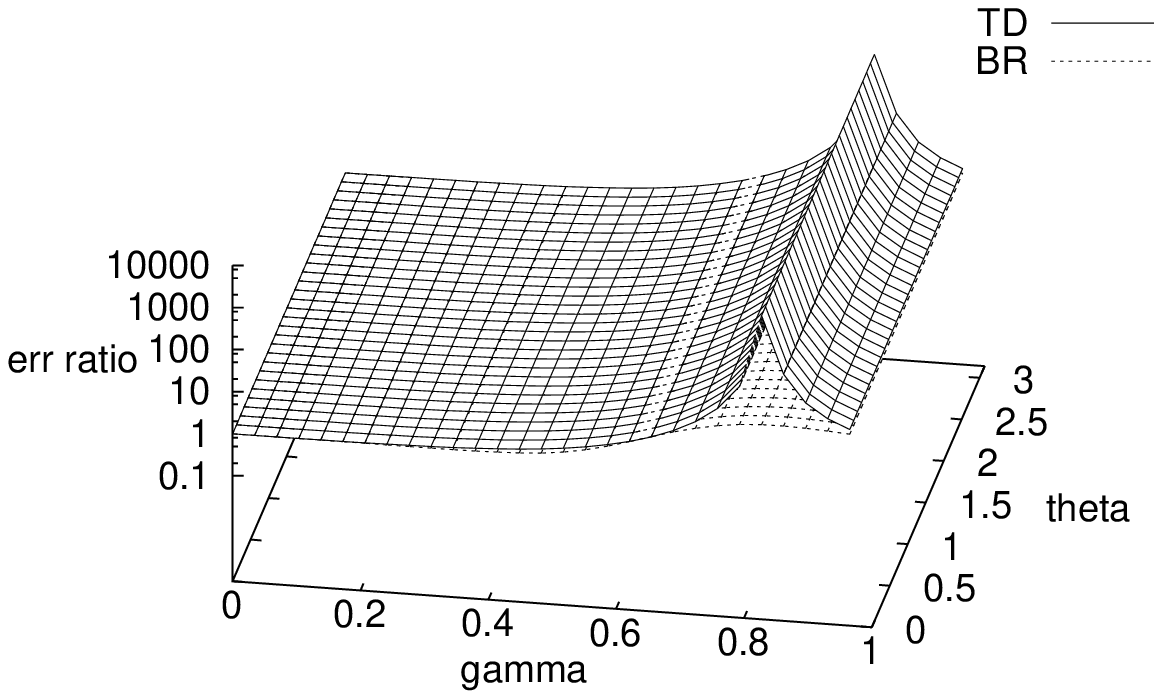}
\end{center}
\end{minipage}
\begin{minipage}[l]{.45\linewidth}
\begin{center}
\includegraphics[width=3.8cm]{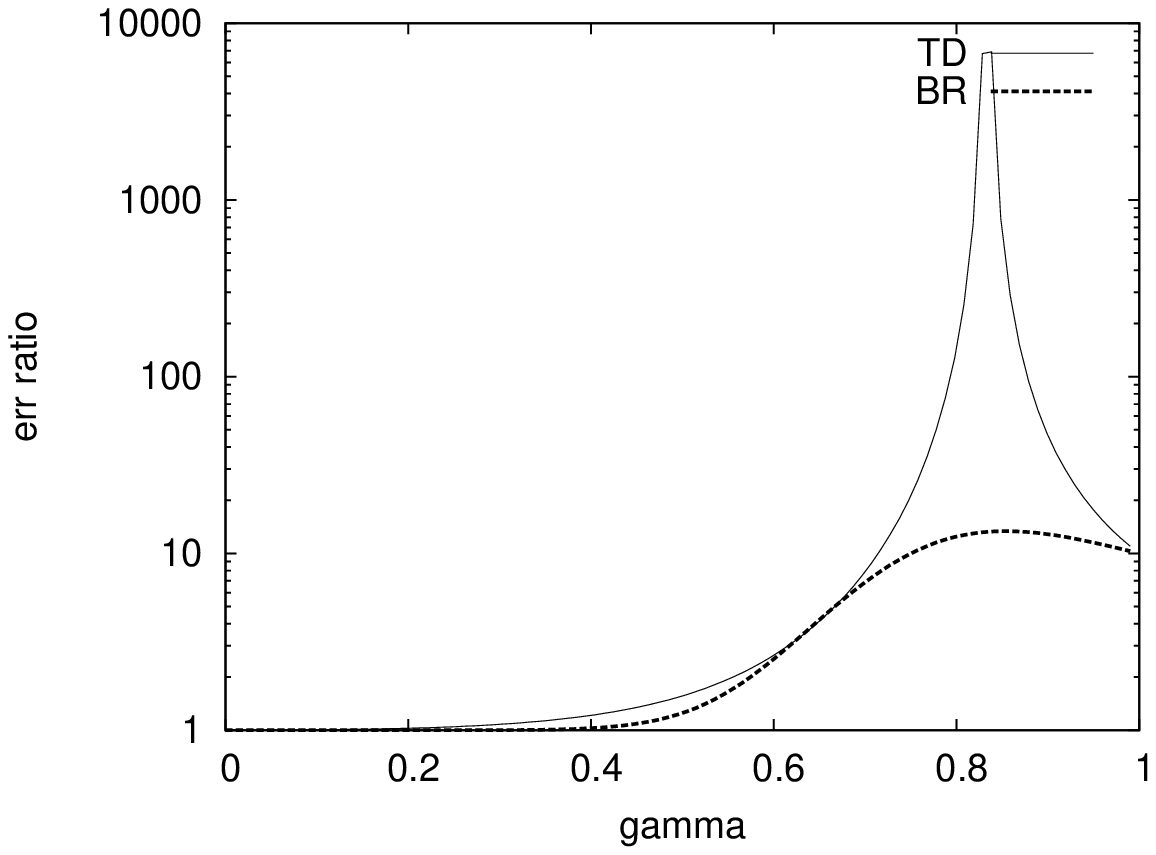}
\end{center}
\end{minipage}
\caption{\label{ex1} Error ratio (in log scale) between the TD/BR projection methods and the best approximation for Example 1, with respect to the discount factor $\gamma$ and the parameter $\theta$ of the reward (Left). It turns out that these surfaces do not depend on $\theta$ so we also draw the graph with respect to $\gamma$ only (Right).
}
\end{figure}
It turns out that these ratios do not depend on $\theta$ (instead of showing this through painful arithmetic manipulations, we will come back to this point and prove it later on). This Figure also displays the graph with respect to $\gamma$ only. We can observe that for any choice of reward function and discount factor, the BR method returns a better value than the TD method. Also, when $\gamma$ is in the neighborhood of $\frac{5}{6}$, the TD error ratio tends to $\infty$ while BR's stays bounded. This Example shows that there exists MDPs where the BR is consistenly better than the TD method, which can give an unbounded error.
One should however not conclude too quickly that BR is \emph{always} better than TD. The literature contains several arguments in favor of TD, one of which is considered in the following Example.

\paragraph{Example 2}

\citet{sutton09} recently described a 3-state MDP example where the TD
method computes the best projection while BR does not. 
The idea behind this 3-state example can be described in a quite general way\footnote{The rest of this section is strongly inspired by a personal communication with Yu.}: Suppose we have a $k+l$-state MDP, of which the Bellman Equation has a block triangular structure:
$
v_1  =  \gamma P_1  v_1 + r_1
~/~
v_2  =  \gamma P_{21} v_1 + P_{22} v_2 + r_2
$
where $v_1 \in \R^k$ and $v_2 \in \R^l$ (the concatenation of the vectors $v_1$ and $v_2$ form the value function). Suppose also that the approximation subspace
$\spn{\Phi}$ is $\R^k \times S_2$ where $S_2$ is a subspace of $\R^l$. For the
first component $v_1$, the approximation space is the entire space $\R^k$.
With TD, we obtain the exact value for the $k$ first components of the value, while with Bellman residual minimization, we do not:
satisfying the first equation exactly is traded for decreasing the error in satisfying the second one (which also involves $v_1$).
In an optimal control context, the
example above can have quite dramatic implications, as $v_1$ can be related to
the costs at some future states accessible from those states associated
with $v_2$, and the future costs are all that matters when making decisions.

Overall, the two methods generate different types of biases, and
distribute error in different manners. In order to gain some more insight, we now turn on to some analytical facts about them.

\section{A Relation and Stability Issues}

\label{stability}

Though several works have compared and considered both methods \cite{schoknecht2002,lagoudakis,munosapi,yu}, the following simple fact has, to our knowledge, never been emphasized \emph{per se}:
\begin{proposition}
\label{reltdbr}
The BR is an upper bound of the TD error, and more precisely:
$$
\forall \hat v \in \spn{\Phi}, E_{BR}(\hat v)^2=E_{TD}(\hat v)^2+\|\T \hat v - \Pi \T \hat v\|_\xi^2.
$$
\end{proposition}
\begin{proof}
This simply follows from Pythagore, as $\Pi T \hat v - T \hat v$ is orthogonal to $\spn{\Phi}$ and  $\hat v-\Pi \T \hat v$ belongs to $\spn{\Phi}$. \qed
\end{proof}
This implies that if one can make the BR  small, then the TD Error will also be small. In the limit case where one can make the BR equal to 0, then the TD Error is also 0.

One of the  motivation for minimizing the BR is historically related to a  well-known result of \citet{baird}:
$
\forall \hat{v},~ \|v-\hat{v}\|_\infty \leq \frac{1}{1-\gamma}\|\T \hat v - \hat v\|_\infty.
$
Since one considers the weighted quadratic norm in practice\footnote{Mainly because it is computationnally easier than doing a max-norm minimization, see however \cite{guestrin} for an attempt of doing max-norm projection.}, the related result\footnote{The proof is a consequence of Jensen's inequality and the arguments are very close to the ones in \citep{munosapi}.} that really makes sense here is:
$
\forall \hat{v},~ \|v-\hat{v}\|_\xi \leq \frac{\sqrt{C(\xi)}}{1-\gamma}\|\T \hat{v} - \hat{v}\|_\xi
$
where $C(\xi):=\max_{i,j}\frac{p_{ij}}{\xi_i}$ is  a ``concentration coefficient'', that can be seen as some  measure of  the stochasticity of the MDP\footnote{If $\xi$ is the uniform law, then there always exists such a $C(\xi) \in (1,N)$ where one recalls that $N$ is the size of the state space; in such a case, $C(\xi)$ is minimal if all next-states are chosen with the uniform law, and maximal as soon as there exists a deterministic transition. See \citep{munosapi} for more discussion on this coefficient.}.
This result shows that it is sound to minimize the BR, since it controls (through a constant) the approximation error $\|v-\hat{v}_{BR}\|_\xi $.

On the TD side, there does not exist any similar result. Actually, the
fact that one can build examples (like Example 1) where the TD
projection is numerically unstable implies that one cannot prove
such a result.  Proposition \ref{reltdbr} allows to understand better
the TD method: by minimizing the TD Error, one only minimizes one part
of the BR, or equivalently this means that one does not care about the
term $\|\T v - \Pi \T v\|_\xi^2$, which may be interpreted as a
measure of adequacy of the projection $\Pi$ with the Bellman operator
$\T$. In Example 1, the approximation error of the TD projection goes
to infinity because this adequacy term diverges. In \cite{csepes}, the authors use an algorithm based on the TD Error and make an assumption on this adequacy term (there called the \emph{inherent Bellman error of the approximation space}), so that their algorithm can be proved convergent.

A complementary view on the potential instability of TD, has been
referred to as a \emph{norm incompatibility issue}
\citep{bertsekas,guestrin}, and can be revisited through the notion of
concentration coefficient.  Stochastic matrices $P$ statisfy
$\|P\|_\infty=1$, which makes the Bellman operator $\T$ $
\gamma$-contracting, and thus its fixed point is well-defined. The
orthogonal projection with respect to $\|\cdot\|_\xi$ is such that
$\|\Pi\|_\xi=1$.  Thus $P$ and $\Pi$ are of norm 1, but for different norms.
Unfortunately,  a general (tight) bound for linear
projections is $\|\Pi\|_\infty \leq \frac{1+\sqrt{N}}{2}$
\citep{minkowski} and it can be shown\footnote{One can prove that
  for all $x$, $\|Px\|_\xi^2 \leq \|x\|_{\xi P}^2 \leq
  C(\xi)\|x\|_\xi^2$. The argument for the first inequality involves
  Jensen's inequality and is again close to what is done in
  \citep{munosapi}.}  that $\|P\|_\xi \leq \sqrt{C(\xi)}$ (which can
thus also be of the order of $\sqrt{N}$).  Consequently, $\|\Pi
P\|_\infty$ and $\|\Pi P\|_\xi$ may be greater than 1, and thus the
fixed point of the projected Bellman equation may not be well-defined.
A known exception where the composition $\Pi P$ has norm 1, is
when one can prove that $\|P\|_\xi=1$ (as for instance when $\xi$ is
the stationary distribution of $P$) and in this case we know from \citet{bertsekas,vanroy} that
\begin{equation}
\label{eqvanroy}
\|v-\hat v_{TD}\|_\xi \leq \frac{1}{\sqrt{1-\gamma^2}} \|v-\vopt\|_\xi.
\end{equation}
Another notable such exception is when
$\|\Pi\|_{max}=1$, as in the so-called ``averager'' approximation
\cite{gordon}. However, in general, the stability of TD is
difficult to guarantee.


%

\section{The unified oblique projection view}

\label{unified}

In the TD approach, we consider finding the fixed point of the
composition of an orthogonal projection $\Pi$ and the Bellman operator
$\T$. Suppose now we consider using a (non necessarily orthogonal)
projection $\Pi$ onto $\spn{\phi}$, that is any linear operator that
satisfies $\Pi^2=\Pi$ and whose range is $\spn{\Phi}$. In their most
general form, such operators are called \emph{oblique projections} and
can be written $\Pi_X = \Phi \pi_X$ with $\pi_X=(X' \Phi)^{-1}X'$. The
parameter $X$ specifies the projection direction: precisely, $\Pi_X$
is the projection onto $\spn{\Phi}$ orthogonally to $\spn{X}$. As for
the orthogonal projections, the following relations $\pi_X\Phi=I$ and
$\pi_X \Pi_X=\pi_X$ hold. Recall that $L=I-\gamma P$. We are ready to state the main result of
this paper:
\begin{proposition}
\label{mainprop}
Write $X_{TD}=\Xi \Phi$ and $X_{BR}=\Xi L \Phi$.
(1) The TD fix point computation and the BR minimization are solutions (respectively with $X=X_{TD}$ and $X=X_{BR}$) of the projected equation $\hat v_X=\Pi_X \T \hat v_X$. (2) When it exists, the solution of this projected equation  is the projection of $v$ onto $\spn{\Phi}$ orthogonally to $\spn{L'X}$, i.e. formally $\hat v_X = \Pi_{L'X}~v$.
\end{proposition}
\begin{proof}
We begin by showing part (2). Writing $\hat v_X=\Phi w_X$, the fixed point equation is:
$
\Phi w_X = \Pi_X (r+\gamma P \Phi w_x).
$ 
Multiplying on both sides by $\pi_X$, one obtains:
$
w_X=\pi_X (r+\gamma P \Phi w_x)
$
and therefore
$
w_X=(I-\gamma \pi_X P \Phi)^{-1} \pi_X r.
$
Using the definition of $\pi_X$, one obtains:
\begin{eqnarray}
w_X & = & (I-\gamma (X' \Phi)^{-1}X' P \Phi)^{-1}(X' \Phi)^{-1}X' r \nonumber\\
& = & \left[(X' \Phi)(I-\gamma (X' \Phi)^{-1}X' P \Phi)\right]^{-1}X' r \nonumber\\
& = & (X' (I-\gamma P) \Phi)^{-1}X'r \label{residual}\\
& = & (X' L \Phi)^{-1}X'L v \nonumber\\
& = & \pi_{L'X} ~ v \nonumber
\end{eqnarray}
where we enventually used $r=L v$. 

The proof of part (1) now follows.
The fact that TD is a special case with $X=\Xi \Phi$ is trivial by construction since then $\Pi_X$ is the orthogonal projection with respect to $\|\cdot\|_\xi$. 
When $X=\Xi L \Phi$, one simply needs to observe from Equations \ref{br} and \ref{residual} and the definition of $\Psi=L\Phi$ that $w_X=w_{BR}$. \qed
\end{proof}

Beyond its nice and simple geometric flavour, a direct consequence of Proposition \ref{mainprop} is that it allows to derive tight error bounds for TD, BR, and any other method for general $X$.
For any square matrix $M$, write $\sigma(M)$ its spectral radius.
\begin{proposition}
\label{propbound}
For any choice of $X$, the approximation error satisfies:
\begin{eqnarray}
\|v-\hat v_X\|_\xi & \leq &\|\Pi_{L'X}\|_\xi \|v - \vopt\|_\xi \label{boundproj}\\
& = & \sqrt{\sigma(ABCB')}\|v - \vopt\|_\xi \nonumber
\end{eqnarray}
where $A =  \Phi' \Xi \Phi$, $B =  (X'L \Phi)^{-1} $ and
$C =  X L \Xi^{-1} L' X$
are matrices of size $m \times m$.
\end{proposition}
Thus, for any $X$, the amplification of the smallest error $\|v - \vopt\|_\xi$ depends on the norm of the associated oblique projection, which can be estimated as the spectral radius of the product of small matrices.
A simple corollary of this Proposition is the following: if the real value $v$
belongs to the feature space $\spn{\Phi}$ (in such a case $v=\vopt$)
then all oblique projection methods find it ($\hat v_X=v$).
\begin{proof}[Proof of Proposition \ref{propbound}~]
Proposition \ref{mainprop} implies that
$
v-\hat v_X = (I-\Pi_{L' X}) v 
 =  (I-\Pi_{L' X}) (I-\Pi_{\Xi \Phi}) v.
$
where we used the fact that $\Pi_{L'X}\Pi_{\Xi\Phi}=\Pi_{\Xi\Phi}$ since $\Pi_{L'X}$ and $\Pi_{\Xi\Phi}$ are projections onto $\spn{\Phi}$. Taking the norm, one obtains
$
\|v-\hat v_X\|_\xi  \leq  \|I-\Pi_{L' X}\|_\xi  \|v - \Pi_{\Xi\Phi}v \|_\xi 
 = \|\Pi_{L'X}\|_\xi  \|v - \vopt\|_\xi
$
where we used the definition of $\vopt$, and the fact that $\|I-\Pi_{L'X}\|_\xi=\|\Pi_{L'X}\|_\xi$ since $\Pi_{L'X}$ is a (non-trivial) projection (see e.g. \cite{Szyld}). Thus Equation \ref{boundproj} holds.

In order to evaluate the norm in terms of small size matrices, one will use the following Lemma on the projection matrix $\Pi_{L'X}=\Phi \pi_{L'X}$:
\begin{lemma}[\citet{yu}]
\label{lemmayu}
Let $Y$ be an $N \times m$ matrix, and $Z$ a $m \times N$ matrix, then
$
\|YZ\|_\xi^2 = \sigma\left((Y' \Xi Y)(Z \Xi^{-1} Z') \right).
$ 
\end{lemma}
Thus,
$
\|\Pi_{L'X}\|^2_\xi  =  \|\Phi \pi_{L'X}\|^2_\xi 
 =  \sigma[ (\Phi' \Xi \Phi)  (\pi_{L'X} \Xi^{-1} (\pi_{L'X})') ] 
 =  \sigma[ \Phi' \Xi \Phi (X' L \Phi)^{-1} X' L \Xi^{-1} L' X (\Phi' L' X)^{-1} ] 
 =  \sigma[ A B C B' ]. \qed
$
\end{proof}

Proposition \ref{mainprop} is closely related to the work of 
\cite{schoknecht2002}, in which the author derived the following characterization of the TD and BR solutions:
\begin{proposition}[\citet{schoknecht2002}]
The TD fix point computation and the BR minimization are orthogonal projections of the value $v$ respectively induced by the seminorm $\|\cdot\|_{Q_{TD}}$\footnote{This is a seminorm because the matrix $Q_{TD}$ is only semidefinite (since $\Phi \Phi'$ has rank smaller than $m<N$). The corresponding projection can still be well defined (i.e. each point has exactly one projection) provided that $\spn{\Phi} \cap \{ x; \|x\|_{Q_{TD}}=0 \} =\{0\}$.} with $Q_{TD}=L' \Xi \Phi \Phi'\Xi L$ and by the norm $\|\cdot\|_{Q_{BR}}$ with $Q_{BR}=L'\Xi L$.
\end{proposition}
This ``orthogonal projection'' characterization and our ``oblique projection'' characterization are in fact equivalent. On the one hand for BR, it is immediate to notice that $\Pi_{\|\cdot\|_{Q_{BR}}}=\Pi_{L'X_{BR}}$. On the other hand for TD, writing $Y=L'X_{TD}$, one simply needs to notice that $\Pi_{L' X_{TD}}=\Pi_Y=\Phi(Y' \Phi)^{-1}Y'=\Phi(Y' \Phi)^{-1}(\Phi' Y)^{-1}(\Phi' Y)Y'=\Phi(\Phi' Y Y' \Phi)^{-1}\Phi' YY'=\Pi_{\|\cdot\|_{Q_{TD}}}$.
The work of \citet{schoknecht2002} suggests that TD and BR are optimal for different criteria, since both look for some $\hat v \in \spn{\Phi}$ that minimizes $\|\hat v - v\|$ for some (semi)norm $\|\cdot\|$. Curiously, our result suggests that neither is optimal, since neither  uses the best projection direction $X^*:=L'^{-1} \Xi  \Phi$ for which $\hat v_{X^*}=\Pi_{L'X^*}v=\Pi_{\Xi \Phi}v=\vopt$ and this supports the empirical evidence that there is no clear ``winner'' between TD and BR. 

Our main results, stated in Propositions \ref{mainprop} and
\ref{propbound}, constitutes a revisit of the work of \citet{yu},
where the authors similarly derived error bounds for TD and BR. Our
approach mimicks theirs: 1) we derive a linear relation between the
projection $\hat v$, the real value $v$ and the best projection
$\vopt$, then 2) analyze the norm of the matrices involved in this
relation in terms of spectral radius of small matrices (through Lemma
\ref{lemmayu}, which is taken from \cite{yu}). From a purely
quantitative point of view, our bounds are identical to the ones
derived there. Two immediate consequences of this quantitative
equivalence are that, as in \cite{yu}, (1) our bound is tight in the
sense that there exists a worst choice for the reward for which it
holds with equality, and (2) it is always better than that of Equation
\ref{eqvanroy} from \citet{bertsekas,vanroy}. However, our work is
qualitatively different: by highlighting the oblique projection
relation between $\hat v$ and $v$, not only do we provide a clear
geometric intuition for both methods, but we also greatly
simplify the form of the results and their proofs (see \cite{yu} for details).

Last but not least, there is globally a significant difference between
our work and the two works we have just mentionned.  The analysis we
propose is unified for TD and BR (and even extends to potential new
methods through other choices of the parameter $X$), while the results
in \cite{schoknecht2002} and \cite{yu} are proved independently for
each method. We hope that our unified approach will help understanding
better the pros and cons of TD, BR, and related alternative
approaches.

\begin{figure}[h]
\begin{minipage}[l]{.49\linewidth}
\begin{center}
{\footnotesize$\gamma=0.9$}
\vspace{-5mm}
\hspace{-5mm}
\includegraphics[width=4.5cm]{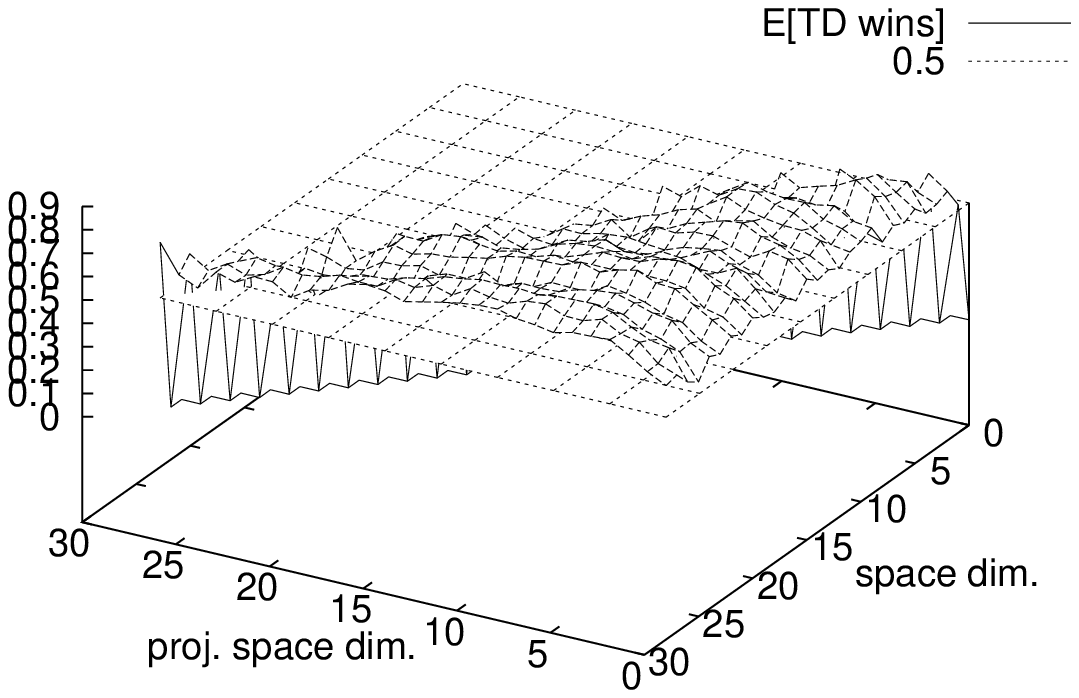}
\end{center}
\end{minipage}
\begin{minipage}[l]{.49\linewidth}
\begin{center}
{\footnotesize$\gamma=0.95$}
\vspace{-5mm}
\hspace{-5mm}
\includegraphics[width=4.5cm]{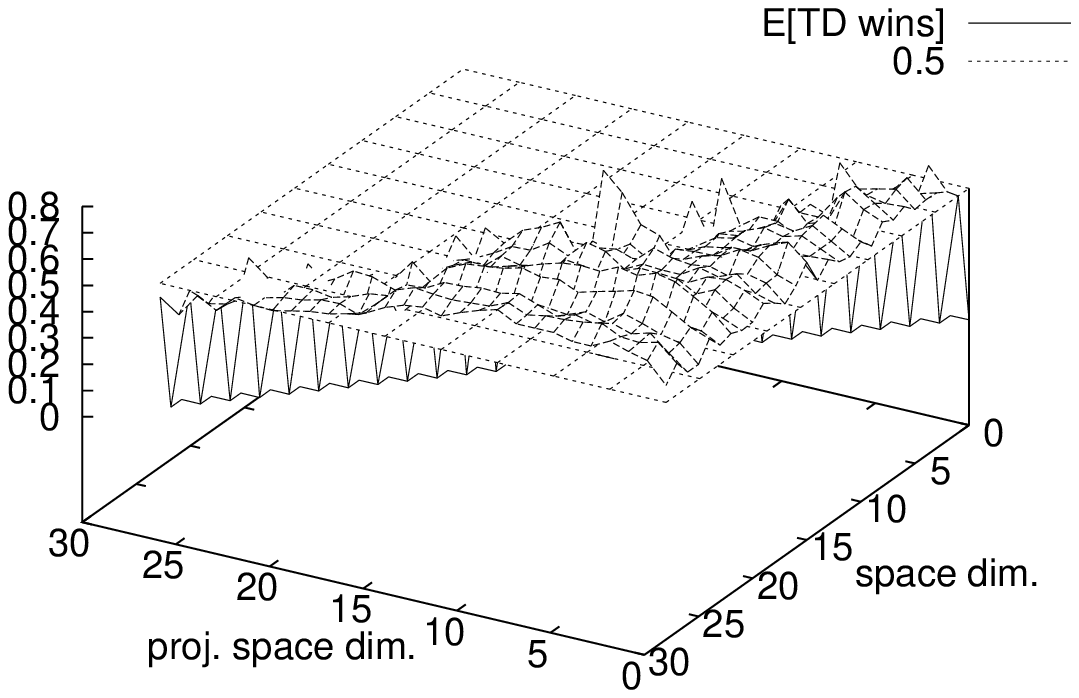}
\end{center}
\end{minipage}

\begin{minipage}[l]{.49\linewidth}
\begin{center}
{\footnotesize$\gamma=0.99$}
\vspace{-5mm}
\hspace{-5mm}
\includegraphics[width=4.5cm]{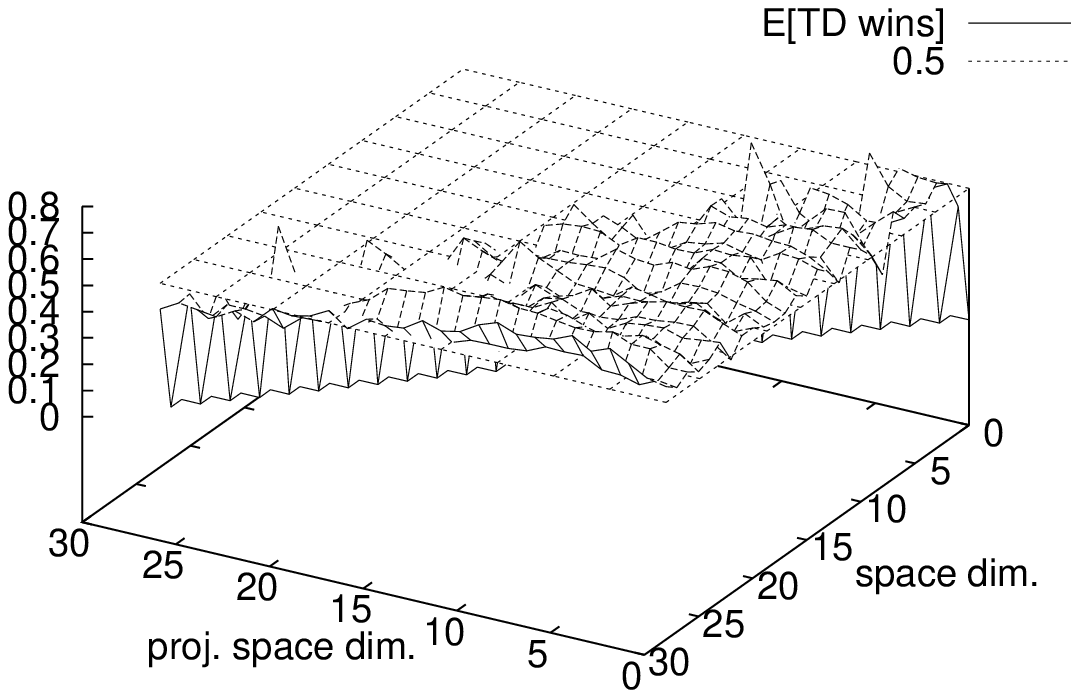}
\end{center}
\end{minipage}
\begin{minipage}[l]{.49\linewidth}
\begin{center}
{\footnotesize$\gamma=0.999$}
\vspace{-5mm}
\hspace{-5mm}
\includegraphics[width=4.5cm]{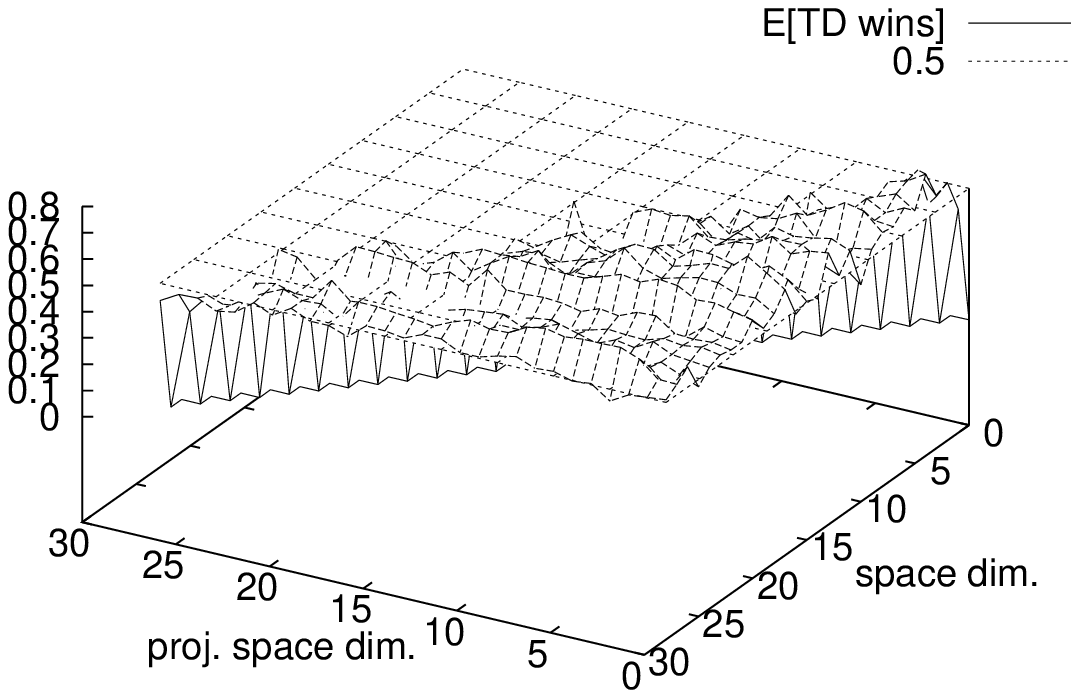}
\end{center}
\end{minipage}
\caption{\label{tdwin}TD win ratio.}
\end{figure}

\begin{figure}[h]
\begin{minipage}[l]{.49\linewidth}
\begin{center}
{\footnotesize$\gamma=0.9$}
\vspace{-5mm}
\hspace{-5mm}
\includegraphics[width=4.5cm]{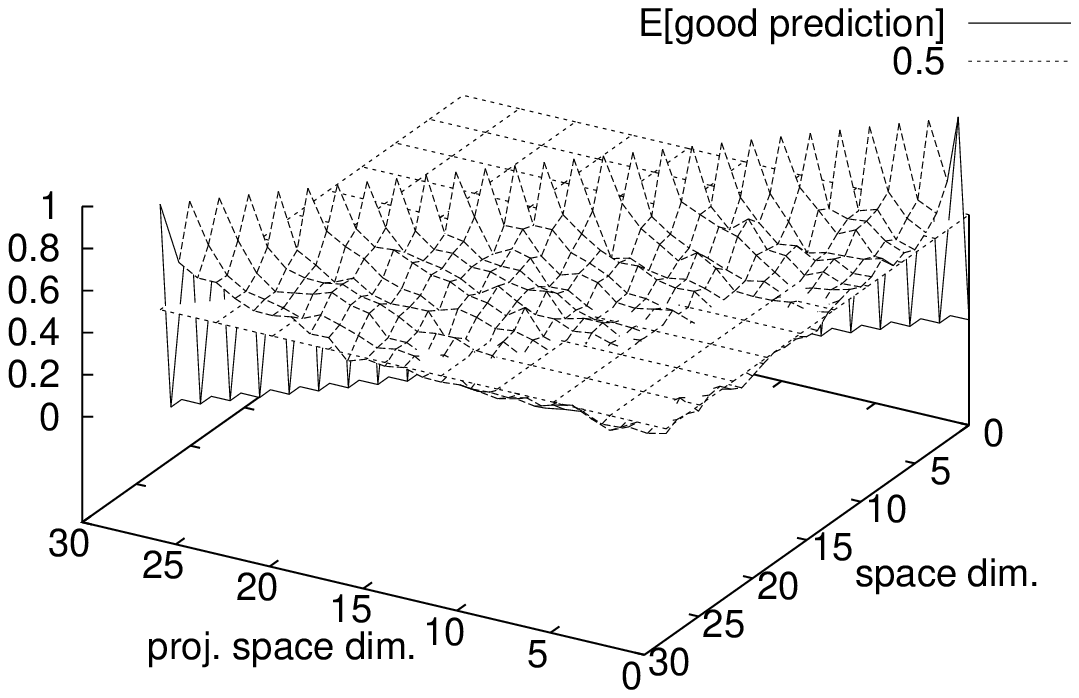}
\end{center}
\end{minipage}
\begin{minipage}[l]{.49\linewidth}
\begin{center}
{\footnotesize$\gamma=0.95$}
\vspace{-5mm}
\hspace{-5mm}
\includegraphics[width=4.5cm]{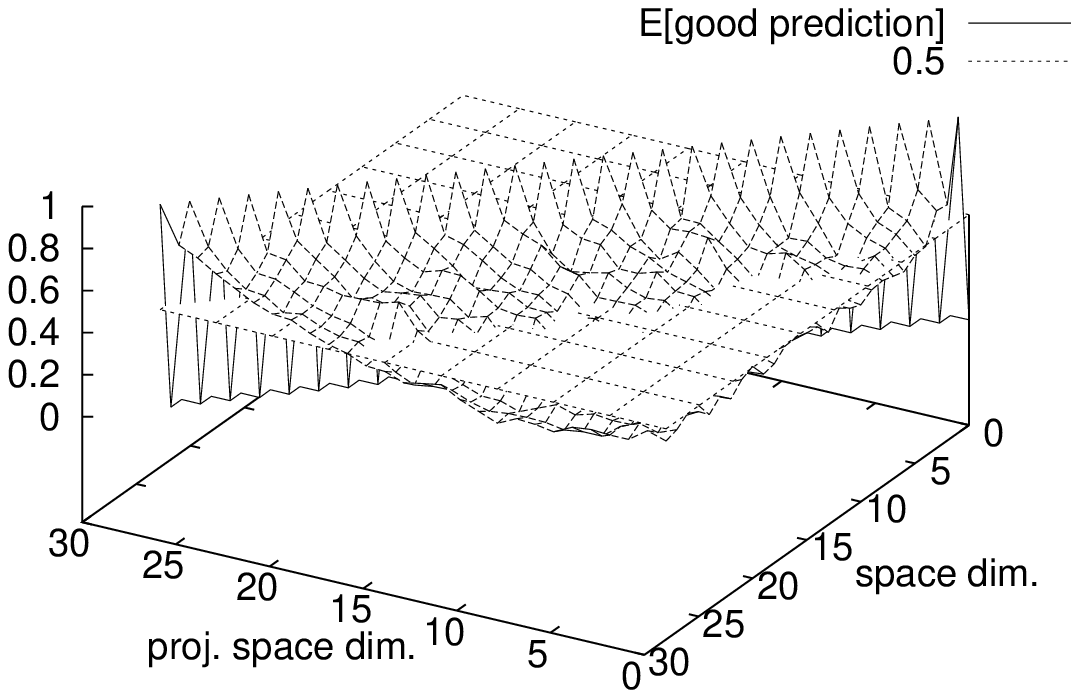}
\end{center}
\end{minipage}

\begin{minipage}[l]{.49\linewidth}
\begin{center}
{\footnotesize$\gamma=0.99$}
\vspace{-5mm}
\hspace{-5mm}
\includegraphics[width=4.5cm]{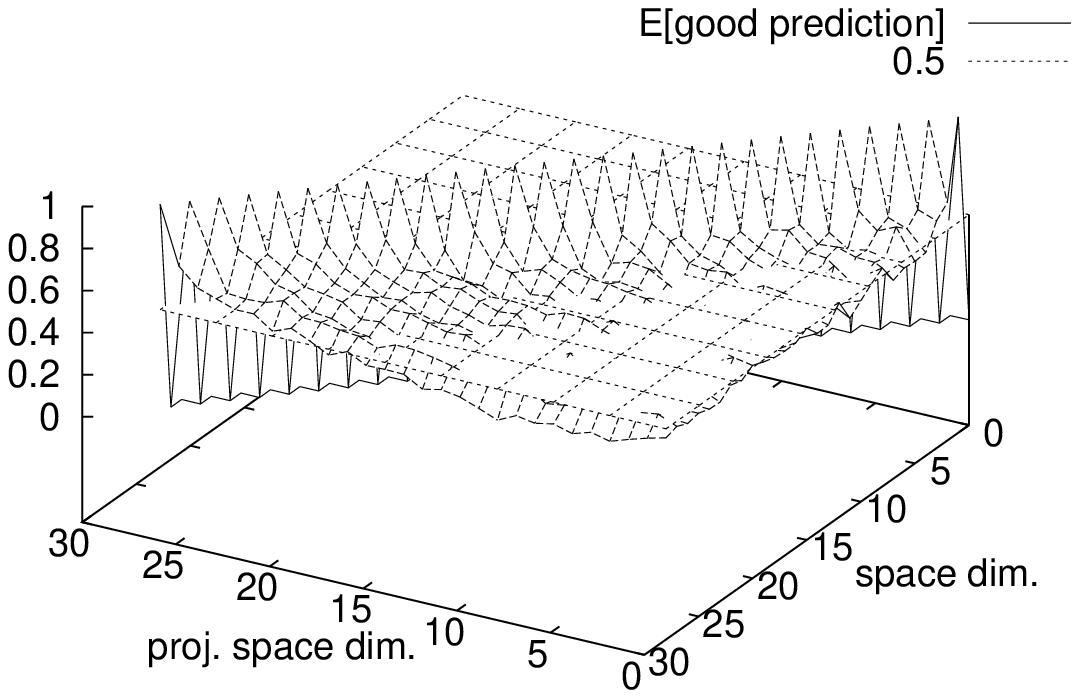}
\end{center}
\end{minipage}
\begin{minipage}[l]{.49\linewidth}
\begin{center}
{\footnotesize$\gamma=0.999$}
\vspace{-5mm}
\hspace{-5mm}
\includegraphics[width=4.5cm]{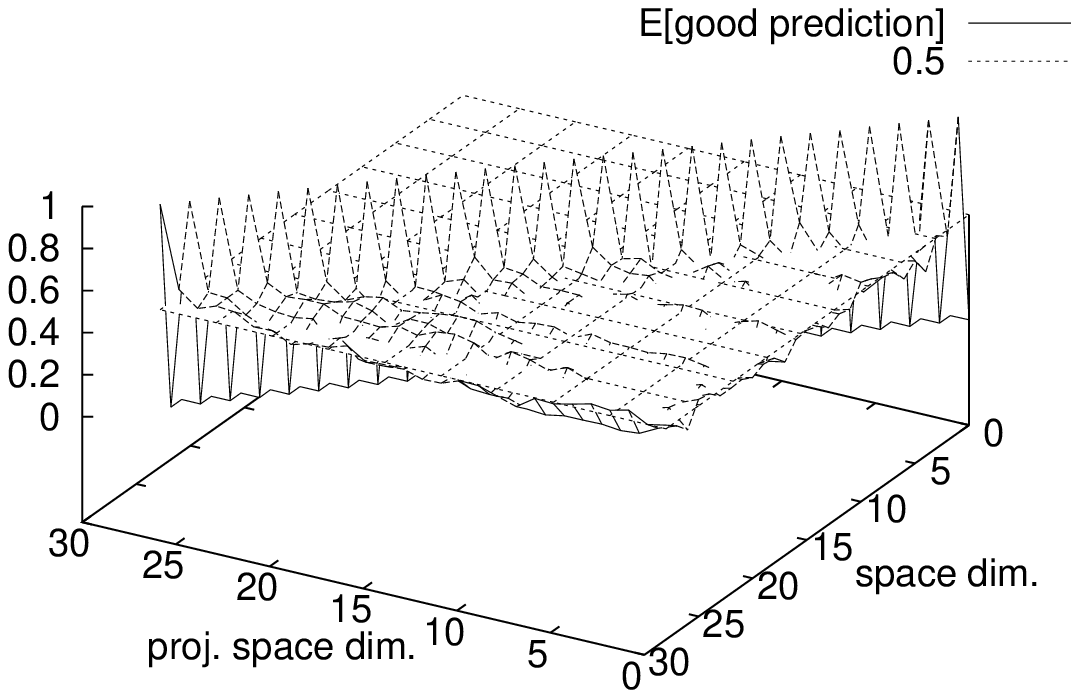}
\end{center}
\end{minipage}
\caption{\label{goodprediction}Prediction of the best method through Prop. \ref{propbound}}
\end{figure}

\begin{figure}[h]

\begin{minipage}[l]{.49\linewidth}
\begin{center}
{\footnotesize$\gamma=0.9$}
\vspace{-5mm}
\hspace{-5mm}
\includegraphics[width=4.5cm]{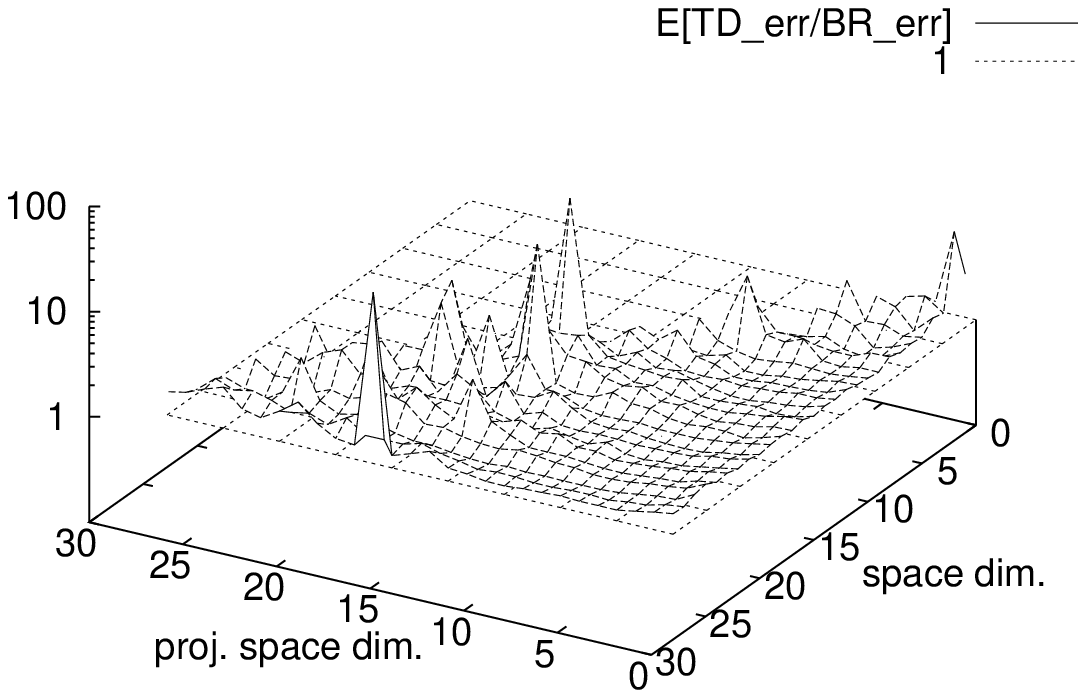}
\end{center}
\end{minipage}
\begin{minipage}[l]{.49\linewidth}
\begin{center}
{\footnotesize$\gamma=0.95$}
\vspace{-5mm}
\hspace{-5mm}
\includegraphics[width=4.5cm]{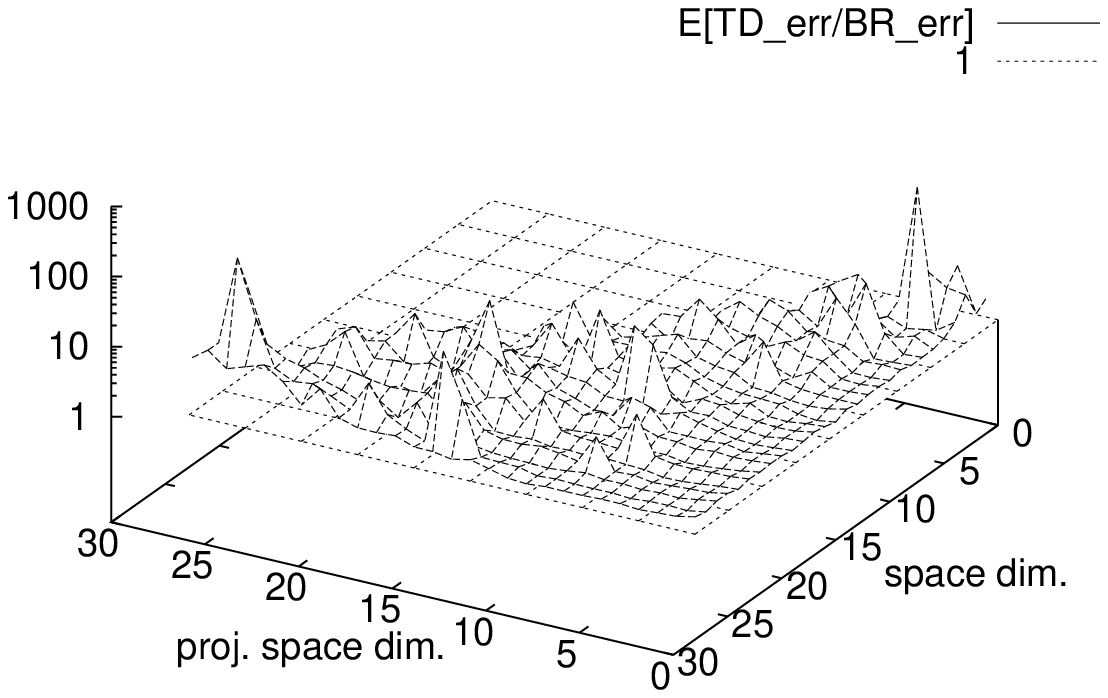}
\end{center}
\end{minipage}

\begin{minipage}[l]{.49\linewidth}
\begin{center}
{\footnotesize$\gamma=0.99$}
\vspace{-5mm}
\hspace{-5mm}
\includegraphics[width=4.5cm]{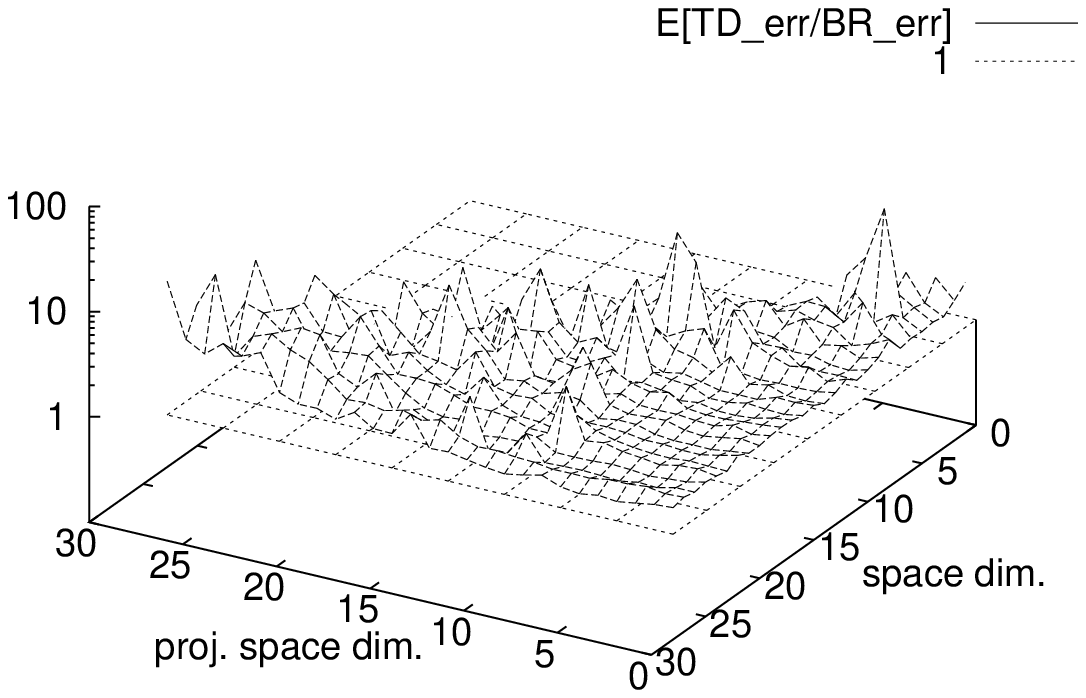}
\end{center}
\end{minipage}
\begin{minipage}[l]{.49\linewidth}
\begin{center}
{\footnotesize$\gamma=0.999$}
\vspace{-5mm}
\hspace{-5mm}
\includegraphics[width=4.5cm]{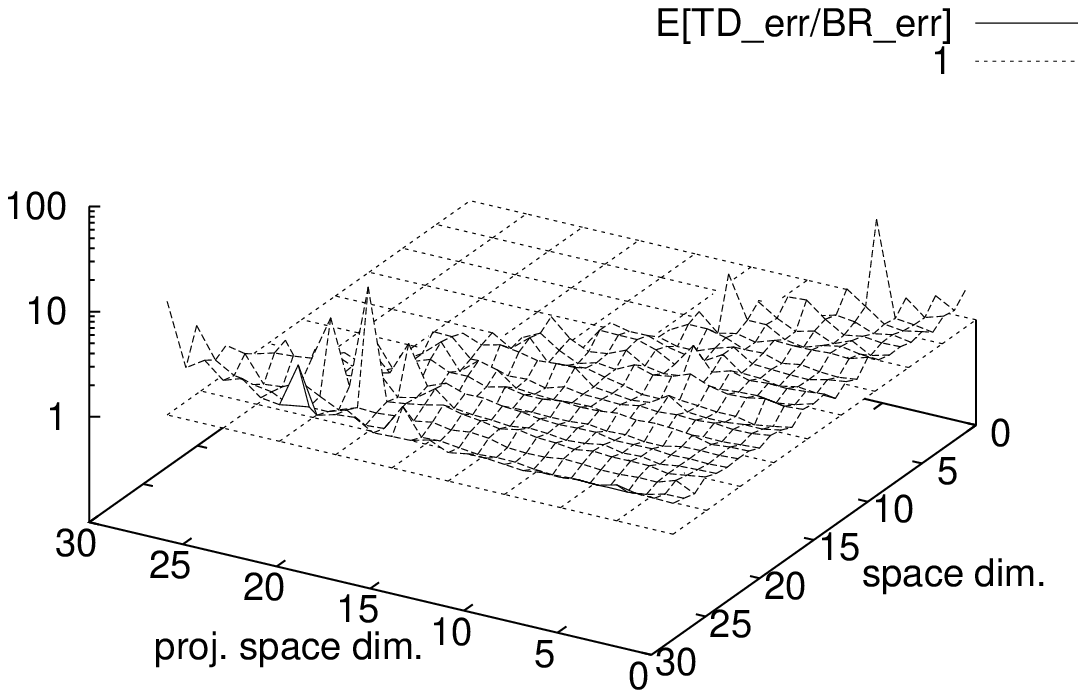}
\end{center}
\end{minipage}
\caption{\label{tdbr}Expectation of $e_{TD}/e_{BR}$.}
\end{figure}

\section{An Empirical Comparison}

\label{experiments}

In order to further compare the TD and the BR projections, we have made some empirical comparison, which we describe now. We consider spaces of
dimensions $n=2,3,..,30$. For each $n$, we consider projections of
dimensions $k=1,2,..,n$. For each $(n,k)$ couple, we generate 20
random projections (through random matrices\footnote{Each entry is a
  random uniform number between -1 and 1.} $\Phi$ of size $(n,k)$ and
random weight vectors $\xi$) and 20 random (uncontrolled) chain like
MDP: from each state $i$, there is a probability $p_i$ (chosen
randomly uniformly on $(0,1)$) to get to state $i+1$ and a probability
$1-p_i$ to stay in $i$ (the last state is absorbing); the reward is a
random vector. For the $20 \times 20$ resulting combinations, we
compute the real value $v$, its exact projection $\vopt$, the TD
fix point $\hat{v}_{TD}$, and the BR projection $\hat{v}_{BR}$. We
then deduce the best error $e=\|v-\vopt\|_\xi$, the TD error
$e_{TD}=\|v-\hat v_{TD}\|_\xi$ and the BR  $e_{BR}=\|v-\hat
v_{BR}\|_\xi$. We also compute the bounds of Proposition \ref{propbound} for both
methods: $b_{TD}$ and $b_{BR}$.  Each such experiment is done for 4 different values of the
discount factor $\gamma$: $0.9$, $0.95$, $0.99$, $0.999$.  

Using this
raw data on $20 \times 20$ problems, we compute for each $(n,k)$ couple some statistics, which we describe now. All the graphs that we display shows the dimension of the space $N$ and of the projected space $m$ on the $x-y$ axes. The $z$ axis correspond to the different statistics of interest.

\begin{figure}[h]
\begin{minipage}[l]{.49\linewidth}
\begin{center}
{\footnotesize$\gamma=0.9$}
\vspace{-5mm}
\hspace{-5mm}
\includegraphics[width=4.5cm]{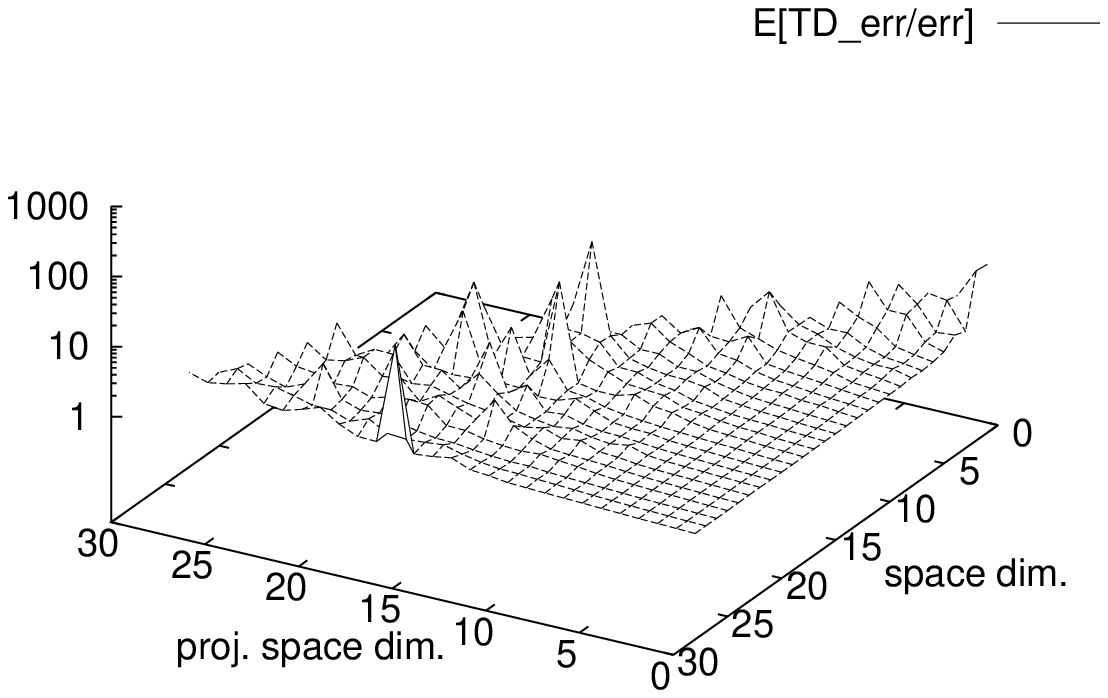}
\end{center}
\end{minipage}
\begin{minipage}[l]{.49\linewidth}
\begin{center}
{\footnotesize$\gamma=0.9$}
\vspace{-5mm}
\hspace{-5mm}
\includegraphics[width=4.5cm]{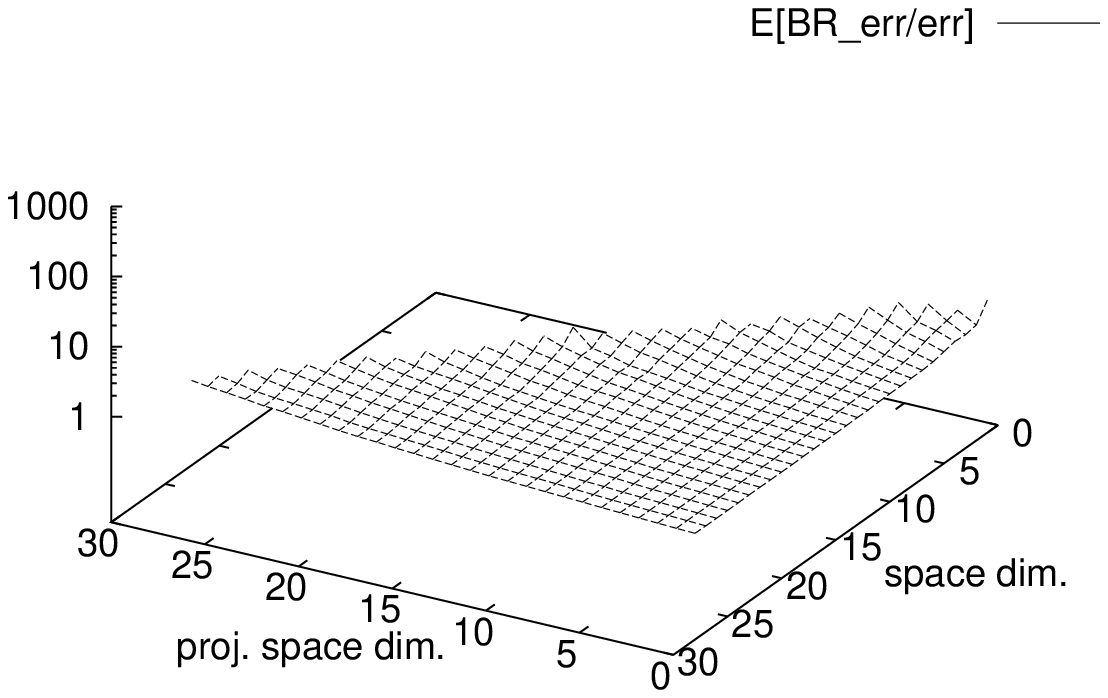}
\end{center}
\end{minipage}

\begin{minipage}[l]{.49\linewidth}
\begin{center}
{\footnotesize$\gamma=0.95$}
\vspace{-5mm}
\hspace{-5mm}
\includegraphics[width=4.5cm]{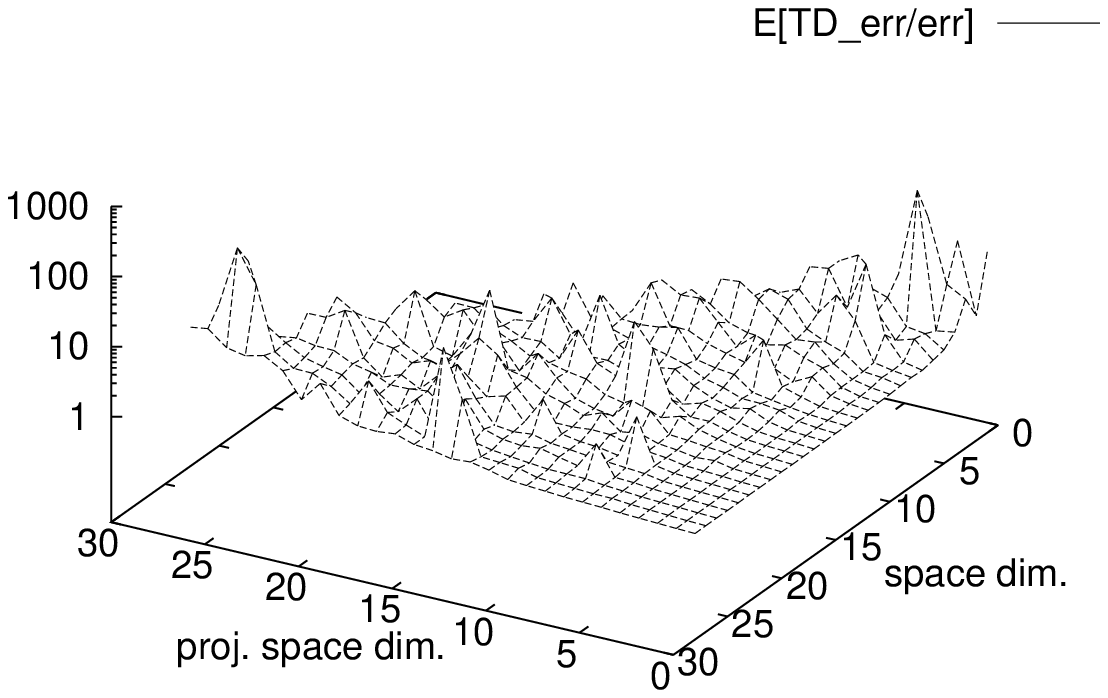}
\end{center}
\end{minipage}
\begin{minipage}[l]{.49\linewidth}
\begin{center}
{\footnotesize$\gamma=0.95$}
\vspace{-5mm}
\hspace{-5mm}
\includegraphics[width=4.5cm]{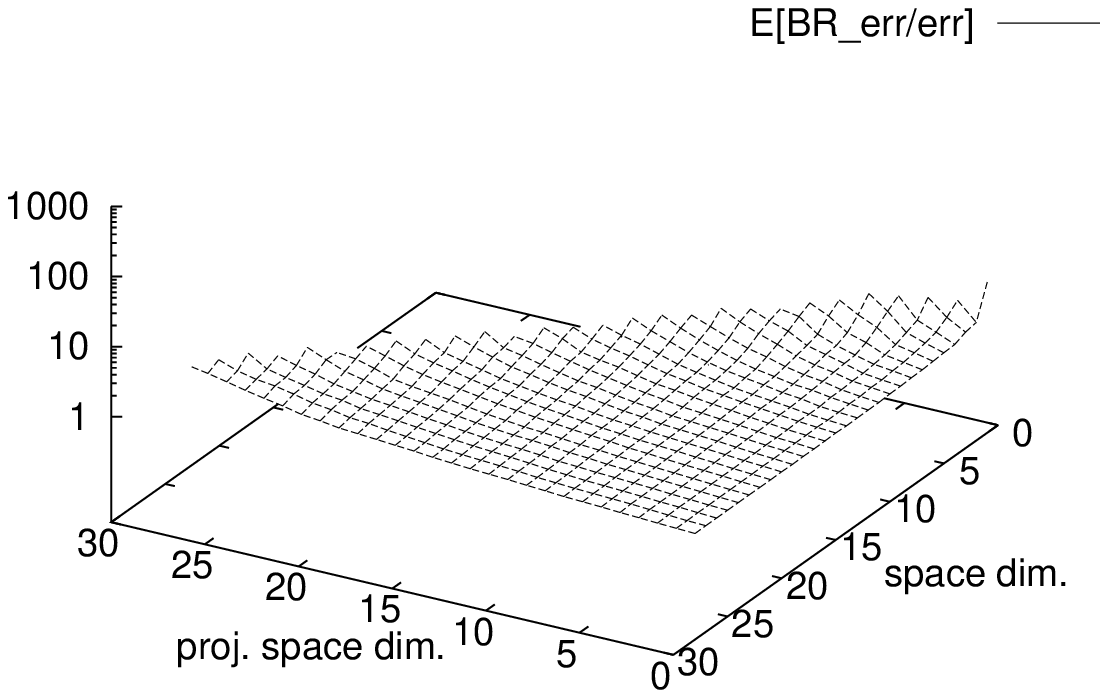}
\end{center}
\end{minipage}

\begin{minipage}[l]{.49\linewidth}
\begin{center}
{\footnotesize$\gamma=0.99$}
\vspace{-5mm}
\hspace{-5mm}
\includegraphics[width=4.5cm]{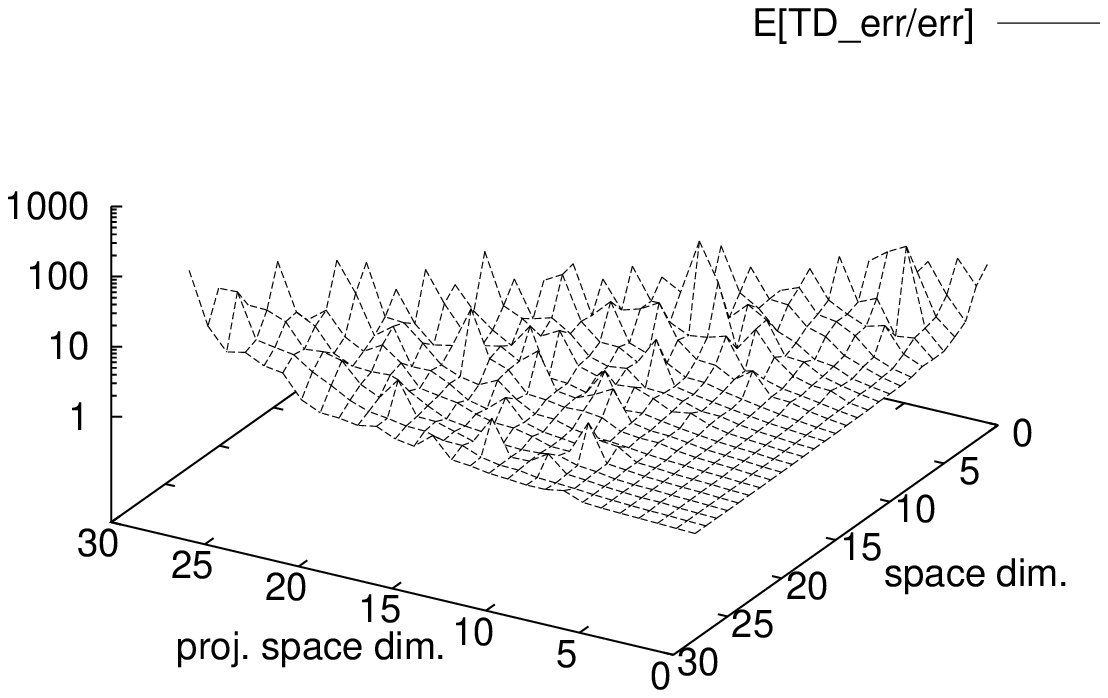}
\end{center}
\end{minipage}
\begin{minipage}[l]{.49\linewidth}
\begin{center}
{\footnotesize$\gamma=0.99$}
\vspace{-5mm}
\hspace{-5mm}
\includegraphics[width=4.5cm]{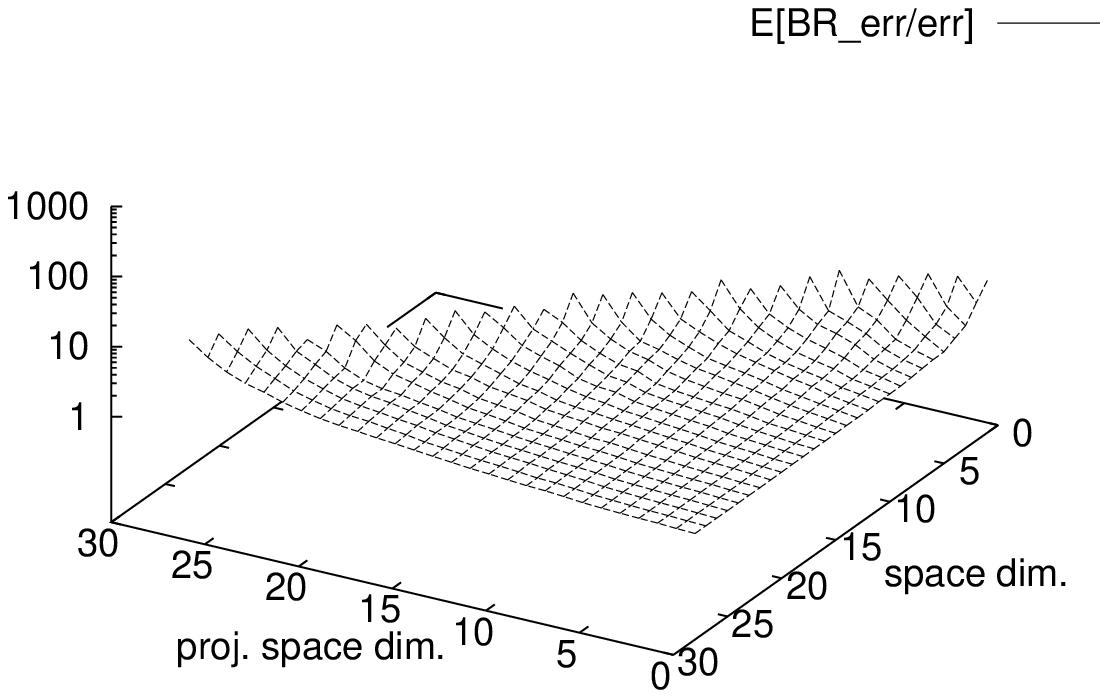}
\end{center}
\end{minipage}

\begin{minipage}[l]{.49\linewidth}
\begin{center}
{\footnotesize$\gamma=0.999$}
\vspace{-5mm}
\hspace{-5mm}
\includegraphics[width=4.5cm]{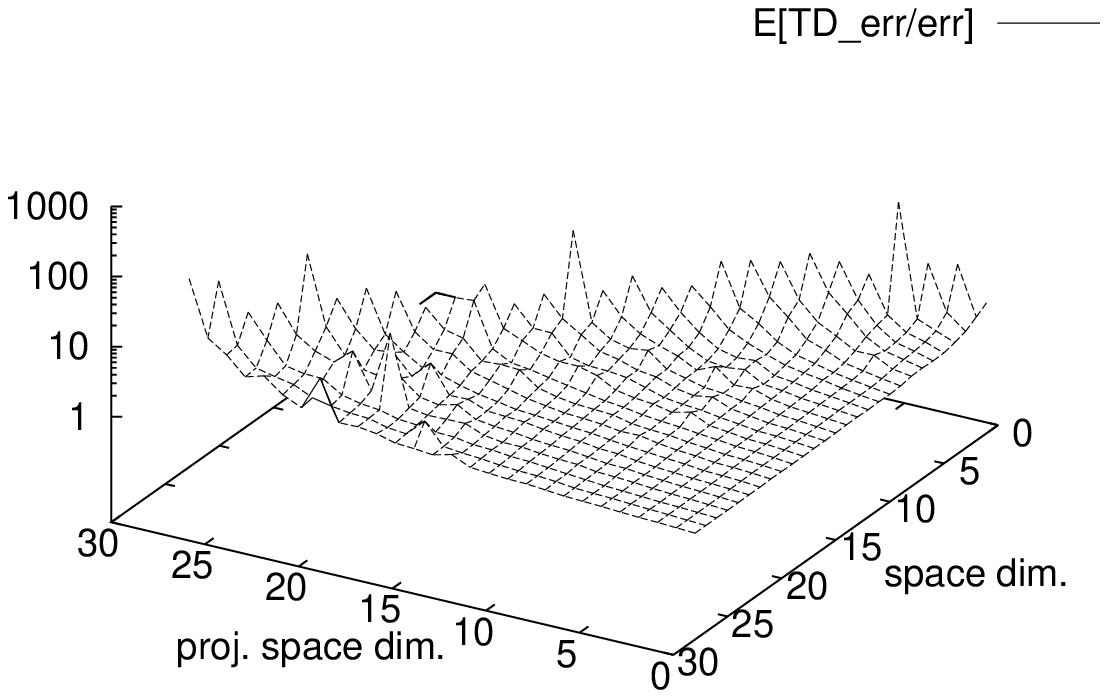}
\end{center}
\end{minipage}
\begin{minipage}[l]{.49\linewidth}
\begin{center}
{\footnotesize$\gamma=0.999$}
\vspace{-5mm}
\hspace{-5mm}
\includegraphics[width=4.5cm]{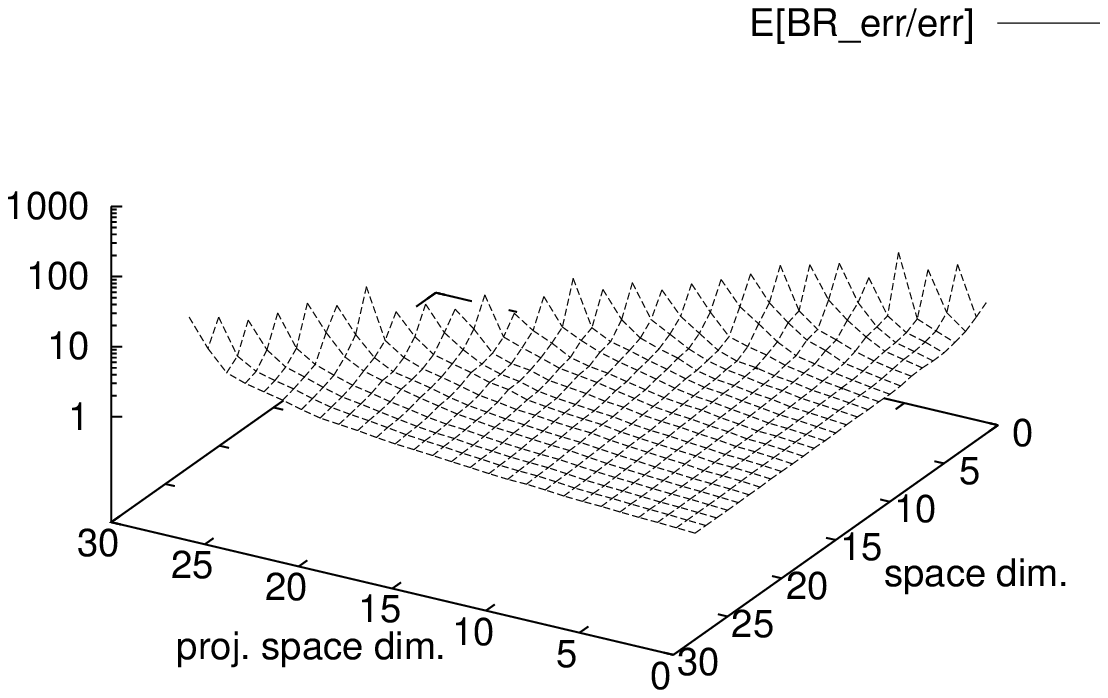}
\end{center}
\end{minipage}
\caption{\label{tderrbrerr}(Left) Expectation of  $e_{TD}/e$ and (Right) of $e_{BR}/e$.}
\end{figure}

Figure \ref{tdwin} shows the proportion of sampled problems where TD method returns a better approximation than BR (i.e. the expectation of the indicator function of  $e_{TD}<e_{BR}$). It turns out that this ratio is consistently greater than $\frac{1}{2}$, which means that the TD method is usually better than the BR method. 
%
Figure \ref{goodprediction} presents the ratio of time the bounds we have presented in Propostion 4 correctly guesses which method is the best (i.e. the expectation of the indicator function of $[e_{TD}<e_{BR}]=[b_{TD}<b_{BR}]$). Unless the feature space dimension is close to the state space dimension, the bounds do not appear very useful for such a decision. 
Figure \ref{tdbr} displays the expectation of $e_{TD}/e_{BR}$. One can observe that, on average, this expectation is bigger than $1$, that is the BR tends to be better, on average, than the TD error. This may look contradictory with our interpretation of Figure \ref{tdwin}, but the explanation is the following: when the BR method is better than the TD method, it is by a larger gap than when it is the other way round. We believe this corresponds to the situation when the TD method in unstable.
Figure \ref{tderrbrerr} allows to confirm this point: it shows the expectation of the relative approximation errors with respect to the best possible error, that is the expectation of  $e_{TD}/e$ and  $e_{BR}/e$. One observes on all charts that this average relative quality of the TD fix point has lots of pikes (corresponding to numerical instabilities), while that of the BR method is smooth.

\section{Conclusion and Future Work}

We have presented the TD fix point and the BR minimization methods for
approximating the value of some MDP fixed policy. We have described
two original examples: in the former, the BR method is consistently
better than the TD method, while the latter (which generalizes the
spirit of the example of \citet{sutton09}) is best treated by TD.
Proposition \ref{reltdbr} highlights the close relation between the
objective criteria that correspond to both methods. It shows that
minimizing the BR implies minimizing the TD error and some extra
``adequacy'' term, which happens to be crucial for numerical
stability.
 
Our main contribution, stated in Proposition \ref{mainprop}, provides
a new viewpoint for comparing the two projection methods, and
potential ideas for alternatives. Both TD and BR can be characterized
as solving a projected fixed point equation and this is to our
knowledge new for BR. Also, the solutions to both methods are some
oblique projection of the value $v$
and this is to our knowledge new for TD and BR. Eventually, this simple geometric
characterization allows to derive some tight error bounds
(Proposition \ref{propbound}).  We have discussed the close relations
of our results with those of \citet{schoknecht2002} and \citet{yu}, and argued that our work simplifies and extends them.
Though apparently new to the Reinforcement Learning community, the very idea
of oblique projections of fixed point equations has been studied in the Numerical Analysis community (see e.g. \citet{saad}). In the future, we plan to study more carefully this literature, and particularly investigate whether it may further contribute to the MDP context.

Concerning the practical question of choosing among the two methods TD
and BR, the situation can be summarized as follows: the BR method is
sounder than the TD method, since the former has a performance
guarantee while the latter will never have one in general. Extensive
simulations (on random chain-like problems of size up to $30$ states,
and for many projection of all the possible space sizes) further
suggest the following facts: (a) the TD solution is more often better
than the BR solution; (b) however sometimes, TD failed dramatically;
(c) overall, this makes BR better on average. Equivalently, one may
say that TD is more risky than BR.

Even if TD  is more risky, there remains several reasons why
one may want to use it in practice, and which our study did not focus
on. In large scale problems, one usually estimates the $m\times m$ linear
systems through sampling. Sampling based methods for BR are more
constraining since they generally require double sampling.
Independently, the fact, highlighted by Propostion
\ref{reltdbr}, that the BR is an upper bound of the TD error, suggests
two things. First, we believe that the variance of the BR problem is higher than that
of the TD problem; thus, given a fixed amount of samples, the TD
solution might be less affected by the corresponding stochastic noise
than the BR one. More generally, the BR problem may be harder to solve
than the TD problem, and from a numerical viewpoint, the latter may provide better solutions. Eventually, we only discussed the TD(0) fix point
method, that is the specific variant of TD($\lambda$) \cite{bertsekas,boyan02} where
$\lambda=0$. Values of $\lambda>0$ solve some of the weaknesses of
TD(0): it can be show that the stability issues disappear for values
of $\lambda$ close to $1$, and the optimal projection $\vopt$ is obtained when
$\lambda=1$. Further analytical and empirical comparisons of
TD($\lambda$) with the algorithms we have considered here (and with
some ``BR($\lambda$)'' algorithm) constitute future research.

Eventually, a somewhat disappointing observation of our study is that
the bounds of Proposition \ref{propbound}, which are
the tightest possible bounds independent of the reward function, did
not prove useful for deciding \emph{a priori} which of the two methods
one should trust better (recall the results showed in Figure
\ref{goodprediction}). Extending them in a way that would take
the reward into account, as well as trying to exploit our original
unified vision of the bounds (Propositions \ref{mainprop} and
\ref{propbound}) are some potential tracks for improvement.

\section*{Acknowlegments}

The author would like to thank Janey Yu for helpful discussions, and the anonymous reviewers for providing comments that helped to improve the presentation of the paper.

\bibliography{biblio}
\bibliographystyle{icml2010}

\end{document}